\documentclass[letterpaper, 10 pt, conference]{ieeeconf}
\IEEEoverridecommandlockouts  %
\usepackage{amsmath}

\usepackage{amsthm}
\usepackage{amssymb}  
\usepackage{mathtools}
\usepackage{bm}

\usepackage{url}
\usepackage{import}
\usepackage{pgf}
\usepackage{pgfplots}
\pgfplotsset{compat=1.18}   %
\usepackage{color}
\usepackage{tabularx} 
\usepackage{multirow}
\usepackage{cite}
\usepackage{algorithm} 
\usepackage{algpseudocode} 
\usepackage{empheq}
\usepackage{multirow}
\usepackage{float}
\usepackage{placeins}

\usepackage[caption=false,font=footnotesize]{subfig}
\usepackage{tikz}
\newcommand{\scalebarimg}[4]{%
    \begin{tikzpicture}[every node/.style={inner sep=0,outer sep=0}]%
        \draw node[name=micrograph] {\includegraphics[width=#2]{#1}};%
        \draw (micrograph.north west) node[anchor=north west,yshift=-1,#4]{\small{#3}};%
    \end{tikzpicture}%
}

\usepackage{balance}

\usepackage{nohyperref}

\newlength{\mycaption}
\newlength{\mySepBetweenFigAndCap}
\newlength{\mySepBetweenTopAndFig}
\newcommand{\ie}{i.e.\ }
\newcommand{\eg}{e.g.\ }

\newcommand{\myvec}[1]{\mathbf{#1}}     %
\newcommand{\bx}{\bm{\xi}}              %
\newcommand{\tbx}{\tilde{\bm{\xi}\vphantom{\bm{\xi}}}}     %

\newtheorem{thm}{Theorem}

\newtheorem{prop}{Proposition}

\theoremstyle{definition}

\newcommand{\IEEEacceptance}{
    \begin{tikzpicture}[overlay, remember picture]
        \path (current page.north east) ++(-2.1,-0.2) node[below left] {
            This paper has been accepted for publication in the 2024 IEEE International Conference on Robotics and Automation.
        };
    \end{tikzpicture}
}

\newcommand\copyrighttext{%
  \footnotesize \textcopyright 2024 IEEE. Personal use of this material is permitted.
  Permission from IEEE must be obtained for all other uses, in any current or future 
  media, including reprinting/republishing this material for advertising or promotional 
  purposes, creating new collective works, for resale or redistribution to servers or 
  lists, or reuse of any copyrighted component of this work in other works.}
\newcommand\copyrightnotice{%
\begin{tikzpicture}[remember picture,overlay]
\node[anchor=south,yshift=10pt] at (current page.south) {\fbox{\parbox{\dimexpr\textwidth-\fboxsep-\fboxrule\relax}{\copyrighttext}}};
\end{tikzpicture}%
}

\title{\LARGE \bf
Learning Barrier-Certified Polynomial Dynamical Systems for Obstacle Avoidance with Robots
}

\author{%
    Martin Schonger$^{1,\dagger}$, %
    Hugo T. M. Kussaba$^{1,\dagger}$, %
    Lingyun Chen$^{1}$, %
    Luis Figueredo$^{2}$,\\%
    Abdalla Swikir$^{1}$,    
    Aude Billard$^{3}$, and %
    Sami Haddadin$^{1}$%
\thanks{This work was funded by the European Union's Horizon 2020 program (grant agreement No. 101070596 - euROBIN and Marie Skłodowska-Curie grant agreement no. 899987), by the German Research Foundation (DFG) as part of Germany’s Excellence Strategy, EXC 2050/1, Project ID 390696704 – Cluster of Excellence ``Centre for Tactile Internet with Human-in-the-Loop'' (CeTI) of Technische Universität Dresden, and by the Bavarian State Ministry for Economic Affairs, Regional Development and Energy (StMWi) (the Lighthouse Initiative KI.FABRIK Bayern Phase 1: Aufbau Infrastruktur).
The authors would also like to thank Katharina Bieker, Philipp Scholl, and Bachir El Khadir for insightful discussions %
}%
\thanks{$^{1}$Munich Institute of Robotics and Machine Intelligence (MIRMI), Technical University of Munich (TUM), Germany. Abdalla Swikir is also with Omar Al-Mukhtar University (OMU), Albaida, Libya. %
}%
\thanks{$^{2}$School of Computer Science, University of Nottingham, UK. Luis Figueredo is also an Associated Fellow at the MIRMI, TUM.}%
\thanks{$^{3}$Learning Algorithms and Systems Laboratory, EPFL, Switzerland.}%
\thanks{$^{\dagger}$The first two authors contributed equally to this work.}%
}

\begin{document}
	
\maketitle
\IEEEacceptance
\copyrightnotice
\thispagestyle{plain}
\pagestyle{plain}

\begin{abstract}
Established techniques that enable robots to learn from demonstrations are based on learning a stable dynamical system (DS).
To increase the robots' resilience to perturbations during tasks that involve static obstacle avoidance, we propose incorporating barrier certificates into an optimization problem to learn a stable and barrier-certified DS. 
Such optimization problem can be very complex or extremely conservative when the traditional linear parameter-varying formulation is used. Thus, different from previous approaches in the literature, we propose to use polynomial representations for DSs, 
which yields an optimization problem that can be tackled by sum-of-squares techniques. 
Finally, our approach can handle obstacle shapes that fall outside the scope of assumptions typically found in the literature concerning obstacle avoidance within the DS learning framework. Supplementary material can be found at the project webpage: \url{https://martinschonger.github.io/abc-ds}
\end{abstract}

\section{Introduction}

The future of robotics will be challenged with novel and complex tasks, involving close proximity to people.
In particular, it will be essential to attain two key objectives. Firstly, since programming such complex tasks requires a high level of expertise and time, enabling the end-user to teach tasks along with preferences of how to perform those tasks is critical. 
Secondly, robots should be able to navigate around obstacles and avoid collisions with objects to perform their tasks efficiently and safely. 

One way to accomplish the first objective is by the \textit{Learning from Demonstration} (LfD) paradigm, that enables robots to learn a task by generalizing from  demonstrations, rather than simply recording and replaying~\cite{BillardCalinonDilmannSchall2008}. The fundamental principle of this paradigm is that it allows end-users to teach robots how to perform new tasks by providing them with examples of a task being performed, eliminating the need for coding on their part. 
An established method to implement LfD in robots is by encoding demonstrations in a stable \textit{dynamical system} (DS)~\cite{BillardMirrazaviFigueroa2022book}. More precisely, the position and velocity of the end effector (or of the joints) are recorded for each demonstration, and optimization algorithms are used to find a globally stable dynamical system which reproduces these demonstrations as closely as possible.

\begin{figure}[!t]%
    \vspace*{0.5em}%
    \centering%
    \includegraphics[width=0.85\columnwidth]{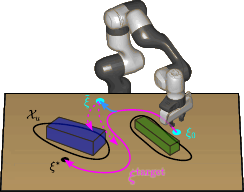}%
    \vspace*{\mySepBetweenFigAndCap}%
    \caption{%
        In a pick-and-place task a robot can be taught how to move from $\bx_0$ to $\bx^*$ by encoding this trajectory in a dynamical system (DS). While the reference trajectory avoids the two obstacles, enclosed by the unsafe set $\mathcal{X}_u$, a disturbance (blue) can push the robot to a state $\tilde{\bx}$ where there is no reference velocity available.
        Nonetheless, it is crucial that the trajectory starting in $\tilde{\bx}$ also does not collide with any obstacle, \ie does not enter $\mathcal{X}_u$.
        For example, the left dashed partial trajectory is unsafe, whereas the right one is safe (up to where it is shown).
        It is desired from the DS to generate safe trajectories for regions of the state space that go beyond the reference trajectories.%
    }%
    \label{fig:visual_abstract}%
    \vspace*{\mycaption}%
\end{figure}

Concurrently, as introduced by the second objective, it is important that robots learn how to realize these demonstrated tasks while also avoiding obstacles.
Encoding demonstrations in a globally stable dynamical system aims to accurately reproduce these for the robot to follow while ensuring that it will still converge towards the desired equilibrium point even if perturbations occur along the trajectory~\cite{KhansariZadehBillard2011LearningStableNonlinear, KronanderBillard2016PassiveInteractionControl,GuptaAradhanaBillard2022LearningHighDim}.
Note that the trajectory followed after a perturbation is not necessarily close to the reference trajectories. This becomes particularly problematic if this new trajectory encounters an obstacle (as illustrated in Fig.~\ref{fig:visual_abstract}). While methods have been proposed to make solutions tend toward the reference trajectory in certain parts of the state space~\cite{FigueroaBillard2022Locallyactiveglobally}, these do not take into account obstacles directly in the learning formulation.

Therefore, when dealing with scenarios involving obstacles, additional methods or techniques should be considered to address the challenges arising from perturbations and to ensure safe obstacle avoidance during the robot's motion.
Considering this, \cite{khansari2012dynamical, huber2019avoidance, huber2022avoiding ,huber2023avoidance}~proposed techniques to modulate DSs to deal with dynamic obstacles.
These techniques are designed to achieve real-time performance. However, they involve altering the original DS, which can potentially result in deviations from the user's initial demonstration.
Many scenarios, however, only require dealing with static obstacles---which allows using offline optimization techniques that \emph{simultaneously} learn a DS encoding the user's demonstrations, and a certificate that ensures that a specific set of trajectories of this DS will avoid user-defined obstacles. 
Since the certified DS results from a single-step optimization problem, it is expected that this DS will better preserve the user's demonstrations compared to two-step methods that first optimize the DS for the user demonstrations and then modify it for avoiding a static obstacle, without taking into account the cost function of the initial step.

One method for incorporating obstacle avoidance as a constraint in the DS optimization problem is by utilizing barrier certificates~\cite{Prajna2004safety}. However, naively combining barrier certificates with the usual Linear Parameter-Varying DS (LPV-DS) formulation~\cite{KhansariZadehBillard2011LearningStableNonlinear, FigueroaBillard2018PhysicalBayesianDS,FigueroaBillard2022Locallyactiveglobally} results in complex optimization problems, even when dealing with simple obstacles. 
Thus, instead of using the typical LPV-DS formulation, we propose to use a polynomial representation for DSs: this approach yields an optimization problem that can be tackled by sum-of-squares (SOS) techniques, and also facilitates the modelling of complex obstacles. 
In particular, this approach enables dealing with obstacles that violate the necessary assumptions of all algorithms previously proposed in the literature on obstacle-avoidance within the DS framework. For instance, \cite{huber2022avoiding}~assumes that individual obstacles must be star-shaped, while the more recent work~\cite{huber2023avoidance} requires that obstacles do not have any holes.
Our method, on the other hand, only requires a basic semi-algebraic description of the obstacles (which can be systematically generated~\cite{dabbene2017simple}). Notably, our method can be used with obstacles that are non-star-shaped (\eg Fig.~\ref{fig:abcds_ushape}) and even with shapes that have holes.

In summary, in this paper we propose using polynomial dynamical systems as a way to combine and synergize the methodology of learning stable dynamical systems with barrier certificates. We call our algorithm \emph{ABC-DS}, which stands for \emph{obstacle Avoidance with Barrier-Certified polynomial Dynamical Systems}, and
to demonstrate its effectiveness, we conducted numerical simulations and experiments. Finally, the source code for this research is publicly accessible.\footnote{%
Available at \href{https://github.com/martinschonger/abc-ds}{https://github.com/martinschonger/abc-ds}}

\section{Preliminaries}\label{sec:preliminaries}

We use the following notation: lowercase bold letters represent vectors in $\mathbb{R}^n$, \eg $\myvec{x}$, uppercase bold letters represent matrices in $\mathbb{R}^{n \times m}$, \eg $\myvec{A}$, and $\| \myvec{x} \|$ denotes the 2-norm of $\myvec{x}$. 
Given a differentiable function ${f\colon \mathbb{R}^n \to \mathbb{R}^m}$, ${\nabla f}$ denotes the gradient of $f$, \ie ${\nabla f(\myvec{x})=(\partial f/\partial x_1 (\myvec{x}), \ldots, \partial f/\partial x_n(\myvec{x}))^\mathsf{T}}$.
Given a polynomial $f$, $\deg(f)$ denotes the degree of $f$. Lastly, ${f^{-1}(0)}$ denotes the $0$-level set of a function ${f:\mathbb{R}^n \to \mathbb{R}}$, that is, ${f^{-1}(0)=\{\myvec{x}\in\mathbb{R}^n : f(\myvec{x})=0\}}$.

\subsection{Learning stable dynamical systems}

Learning a DS with a globally asymptotically stable equilibrium point in ${\bx^{\star}\in\mathbb{R}^n}$ can be expressed as an optimization problem whose variables are the dynamics\footnote{Hereafter we assume that ${f\colon\mathbb{R}^n \to \mathbb{R}^n}$ is such that the DS ${\dot{\bx} = f(\bx)}$ has a unique solution for each initial condition ${\bx(0) \in \mathbb{R}^n}$.} ${f\colon\mathbb{R}^n\to\mathbb{R}^n}$ and a radially unbounded Lyapunov function ${V\colon\mathbb{R}^n\to\mathbb{R}}$. 
More precisely, given reference data in the form of $N$ tuples of position and velocity, ${(\bx^{(i)}_{\mathrm{ref}}, \dot{\bx}^{(i)}_{\mathrm{ref}})}$, ${i=1,\ldots,N}$, one needs to solve the following optimization problem (cf.~\cite{KhansariZadehBillard2011LearningStableNonlinear}):\footnote{%
We assume that the sequence $(\bx^{(i)}_{\mathrm{ref}})_{i=1}^{N}$ is made of distinct elements, and that ${(\bx^{(N)}_{\mathrm{ref}}, \dot{\bx}^{(N)}_{\mathrm{ref}}) = ({\bx^{\star}}, \myvec{0})}$. Moreover, hereafter it is assumed that ${\bx^{\star} = \myvec{0}}$ without loss of generality since one can use the reference data ${(\bx^{(i)}_{\mathrm{ref}} - {\bx^{\star}}, \dot{\bx}^{(i)}_{\mathrm{ref}})_{i=1}^N}$ instead of the original data, and translate the obtained DS in such way that its equilibrium point is ${\bx^{\star}}$.
}
\begin{subequations}\label{eq:DS_opt_problem}
\begin{alignat}{2}
&\!\min_{f}  &\quad& {\sum_{i=1}^{N}\nolimits \| \dot{\bx}^{(i)}_{\mathrm{ref}} - f(\bx^{(i)}_{\mathrm{ref}}) \|^2} \label{eq:DS_opt_problem_objective}\\
&\text{\textup{s.t.}}  &&  f(\myvec{0}) \, = 0,\label{eq:DS_opt_problem_constraints_1}\\
&  &&  V(\myvec{0}) = 0, \ V(\bx) > 0 %
    \text{ for all } \bx \in \mathbb{R}^n {\setminus} \{ \myvec{0} \},\label{eq:DS_opt_problem_constraints_2}\\
&  &&  \dot{V}(\myvec{0}) = 0, \ \dot{V}(\bx) < 0 %
    \text{ for all } \bx \in \mathbb{R}^n {\setminus} \{ \myvec{0} \}.\label{eq:DS_opt_problem_constraints_3}
\end{alignat}
\end{subequations}
Finding $f(\bx)$ and $V(\bx)$ for each ${\bx\in\mathbb{R}^{n}}$ is an infinite-dimensional optimization problem and, as such, not computationally tractable unless
a suitable parameterization of $f$ and $V$ is chosen such that
\eqref{eq:DS_opt_problem} becomes a finite-dimensional approximation.
In particular, previous approaches in the literature propose \textit{linear parameter-varying} (LPV) dynamics $f$:
\begin{equation}\label{eq:func_parameter_varying_grad_system}
    f(\bx) = \sum_{k=1}^K\nolimits \gamma_k (\bx) \, \myvec{A}_k \, \bx,
\end{equation}
where ${\myvec{A}_k\in\mathbb{R}^{n \times n}}$ and ${\gamma_k\colon \mathbb{R}^n\to\mathbb{R}}$ is the $K$-component Gaussian mixture model (GMM) defined as 
\begin{equation}\label{eq:GMM_gamma}
\gamma_k(\myvec{x}) \coloneqq \frac{\pi_k \, p(\myvec{x}|k)}{\sum_j \pi_j \, p(\myvec{x}|j)},    
\end{equation}
with ${\pi_k\ge 0}$ being the mixing weights satisfying ${\sum_{k=1}^{K} \pi_k = 1}$, and %
${p(\myvec{x}|k)}$, ${k=1,\ldots,K}$, being the probability density function of a multivariate normal distribution. %
The DS corresponding to \eqref{eq:func_parameter_varying_grad_system} is given by
\begin{equation}\label{eq:LPV_DS}
    \dot{\bx}(t) = \sum_{k=1}^K\nolimits \gamma_k(\bx) \, \myvec{A}_k \, \bx,
\end{equation}
and the stability of the origin of \eqref{eq:LPV_DS} is certified by either (i) using the Lyapunov function ${V(\bx) = \|\bx\|^2}$, as in~\cite{KhansariZadehBillard2011LearningStableNonlinear}; or (ii) searching for a generic quadratic Lyapunov function, \ie ${V(\bx) = \bx^\mathsf{T} \myvec{P} \, \bx}$, where ${\myvec{P}\in\mathbb{R}^{n \times n}}$ is positive-definite, as in~\cite{FigueroaBillard2018PhysicalBayesianDS}.

\vspace{0.1em}

\subsection{Barrier certificates}
Barrier certificates were introduced in~\cite{Prajna2004safety} with the purpose of verifying the safety of control systems. 
Given an \textit{initial set} ${\mathcal{X}_0 \subset \mathbb{R}^n}$ and an \textit{unsafe set} ${\mathcal{X}_u \subset \mathbb{R}^n}$,
the DS is \emph{safe} if ${\bx_{0} \in \mathcal{X}_0}$ implies that the solution of the initial value problem ${\vphantom{\dot{\dot{\dot{\dot{\dot{m}}}}}}\dot{\tbx}(t) = f(\tbx(t))}$ with ${\tbx(0) = \bx_{0}}$
satisfies ${\tbx(t)\not\in\mathcal{X}_u}$ for all ${t \ge 0}$ (see Fig.~\ref{fig:robot_barrier_detailed}).

To verify if a given system is safe, one can compute a barrier certificate. 
Similar to Lyapunov functions, barrier certificates do not necessitate the explicit computation of system trajectories, offering a convenient way of integrating safety constraints as additional constraints in optimization problems.
The next theorem is adapted from~\cite{prajna2007framework} and shows how to employ barrier certificates:\footnote{%
Instead of checking a strict inequality for the term $\dot{B}$ over the 0-level set of $B$ as in Proposition~3 of~\cite{prajna2007framework}, we propose to check a non-strict inequality over a larger set containing the 0-level set of $B$.
While Theorem~1 of the previous work~\cite{Prajna2004safety} proposed checking this non-strict inequality only over the 0-level set of $B$, a variation of the counterexample in~\cite{she2015safety} shows that~\cite[Theorem 1]{Prajna2004safety} is incorrect.
}
\begin{thm}\label{thm:barrier_fcn}
Let ${\mathcal{X}_0, \mathcal{X}_u \subset \mathbb{R}^n}$, and ${\varepsilon_1 > 0}$.
If there exists a continuously differentiable function ${B\colon \mathbb{R}^n \to \mathbb{R}}$ such that
\begin{subequations}\label{eq:B_constraint}
\begin{alignat}{2}
& B(\bx) \le 0,\ && \text{ for all } \bx \in \mathcal{X}_0,\label{eq:B_constraint_1}\\
& B(\bx) > 0,\ && \text{ for all } \bx \in \mathcal{X}_u,\label{eq:B_constraint_2}\\
& {\nabla{B}(\bx)}^\mathsf{T} f(\bx) \le 0,\ &\quad& \text{ for all } \bx  \text{ with } |B(\bx)| \le \sqrt{\varepsilon_1},\label{eq:B_constraint_3}
\end{alignat}
\end{subequations}
then for all trajectories $\bx(t)$ of the system 
${\dot{\bx}(t) = f(\bx(t))}$ such that ${\bx(0)\in\mathcal{X}_0}$, one has that ${\bx(t)\not\in\mathcal{X}_u}$ for all ${t \ge 0}$.
\end{thm}%
\begin{proof}
   Suppose ${\bx_{0} \in \mathcal{X}_0}$ and there exists $B$ satisfying \eqref{eq:B_constraint}. If the DS is not safe, there exist a trajectory ${\tbx:[0,+\infty)\to\mathbb{R}}$ starting at $\bx_{0}$ and ${t' > 0}$ such that ${B(\tbx(t')) > 0}$ by \eqref{eq:B_constraint_2}. 
   By \eqref{eq:B_constraint_1}, ${B(\tbx(0)) \le 0}$. Since the function ${t \mapsto B(\tbx(t))}$ is continuous, it follows from the intermediate value theorem that there exists $t_1$ with ${0\le t_1 < t'}$  and ${{-\sqrt{\varepsilon_1}} \le B(\tbx(t_1) \le 0}$, and exists $t_2$ with ${t_1 < t_2 \le t'}$ and ${0 < B(\tbx(t_2) \le \sqrt{\varepsilon_1}}$.
   Since the function ${t \mapsto B(\tbx(t))}$ is differentiable, the mean value theorem assures there exists a $t''$ in ${(t_1, t_2)}$ such that 
   ${\frac{d}{dt}[B(\tbx(t)]}$
   evaluated at $t''$ is strictly positive. But this contradicts \eqref{eq:B_constraint_3}.
\end{proof}

A function $B$ satisfying conditions \eqref{eq:B_constraint_1}--\eqref{eq:B_constraint_3}  of Theorem~\ref{thm:barrier_fcn} is called a \textit{barrier certificate} for $f$ with respect to $\mathcal{X}_0$ and $\mathcal{X}_u$.
While this theorem only makes a direct statement regarding safety of trajectories starting in $\mathcal{X}_0$, the safe region of the state space actually amounts to all points $\bx$ where ${B(\bx) \le 0}$ and can be much larger than $\mathcal{X}_0$ (see Fig.~\ref{fig:robot_barrier_detailed}). We define the \emph{certified safe set} %
as ${\mathcal{X}_s \coloneqq \{\bx \in \mathbb{R}^n : B(\bx) \le 0\}}$. By \eqref{eq:B_constraint_2} and \eqref{eq:B_constraint_3}, any trajectory starting in $\mathcal{X}_s$ will never cross the level set $B^{-1}(0)$ and, therefore, never reach $\mathcal{X}_u$.
The reason for taking $\mathcal{X}_0$ as a strict subset of $\mathcal{X}_s$ is to increase the search space of the optimization problem.

While Theorem~\ref{thm:barrier_fcn} is very flexible and gives many options for searching a barrier certificate  $B$, it is desired to restrict $B$ to a set of functions that turns conditions \eqref{eq:B_constraint_1}--\eqref{eq:B_constraint_3} of Theorem~\ref{thm:barrier_fcn} into a computationally tractable optimization problem.
A common approach for achieving this is by the use of sum-of-squares (SOS) optimization~\cite{Prajna2004safety}.%

\begin{figure}[!h]%
    \vspace*{0.36em}%
    \centering%
    \includegraphics[width=1.0\columnwidth]{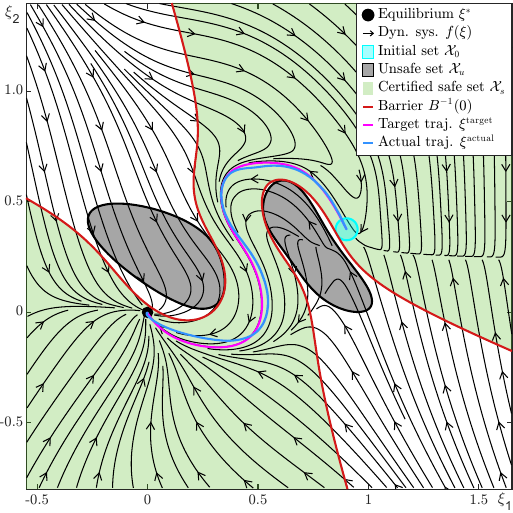}%
    \vspace*{\mySepBetweenFigAndCap}%
    \caption{%
        A dynamical system $f$ is safe if none of its trajectories starting from a state in the initial set $\mathcal{X}_0$ reach a state in the unsafe set $\mathcal{X}_u$. In the particular case of obstacle avoidance, $\mathcal{X}_u$ can be specified to enclose any obstacles.
        The certified safe set $\mathcal{X}_s$ of $f$ with respect to a barrier certificate $B$ is a superset of $\mathcal{X}_0$ and amounts to all states $\bx$ for which ${B(\bx) \le 0}$. When starting at any state in $\mathcal{X}_s$, the trajectory generated by $f$ is guaranteed to not enter $\mathcal{X}_u$.
        The specific DS shown in the figure was generated by our proposed approach, with reference data obtained from a robot. The trajectory generated by the DS when starting in the center of $\mathcal{X}_0$ is plotted in pink. The trajectory executed by the robot when controlled by this DS is shown in blue. %
    }%
    \label{fig:robot_barrier_detailed}%
    \vspace*{\mycaption}%
\end{figure}

\section{Proposed method}

The optimization framework \eqref{eq:DS_opt_problem} offers the flexibility of combining the constraints presented in 
\eqref{eq:DS_opt_problem_constraints_1}--\eqref{eq:DS_opt_problem_constraints_3} and \eqref{eq:B_constraint} into a single optimization problem minimizing the objective \eqref{eq:DS_opt_problem_objective}.
However, the computational feasibility of the resulting optimization problem depends on the $f$ used to represent the DS. 
For instance, using an LPV dynamics such as in 
\eqref{eq:func_parameter_varying_grad_system}
leads to a highly complex optimization problem due to the exponential terms $\gamma_k$ in \eqref{eq:GMM_gamma}. 

Here, we propose to use a polynomial dynamics $f$ for parameterizing DSs to enable the formulation of an optimization problem computationally tractable by sum-of-squares techniques. Simultaneously, this parameterization has the capability to capture complex dynamics: as proved by~\cite{AhmadiKhadir2023LDSwithSideInformation}, trajectories of a continuously differentiable vector field that pass through pre-fixed points can be arbitrarily approximated by trajectories of polynomial vector fields that pass through the same points. Furthermore, experiments presented in~\cite{abyaneh2023learning}, as well as those detailed in Section~\ref{sec:simulations} of our paper,\footnote{%
The results of~\cite{abyaneh2023learning} regarding polynomial DSs as an alternative to LPV-DS were developed concurrently with ours. However, the primary focus of this paper is using polynomial DSs as a way to leverage the use of barrier certificates, while~\cite{abyaneh2023learning} focused on the comparison of polynomial DSs with LPV-DS (without barrier certificates).}
demonstrate that when applied in the context of learning DSs, their performance is on par with LPV-DS.

We show now how to transform Theorem~\ref{thm:barrier_fcn} into a SOS optimization problem.
For that, we assume that the initial set $\mathcal{X}_0$ and the unsafe set $\mathcal{X}_u$ are \emph{basic semialgebraic sets}, that is, they can be described as 
\begin{subequations}
\begin{alignat}{2}
& \mathcal{X}_0 &=& \left\{ \bx \in \mathbb{R}^n : \tilde{g}_1(\bx) \ge 0, \ldots,
\tilde{g}_{p}(\bx) \ge 0 \right\}, \label{eq:initial_set_semialg}\\
& \mathcal{X}_u &=& \left\{ \bx \in \mathbb{R}^n : {g}_1(\bx) \ge 0, \ldots,
{g}_{m}(\bx) \ge 0 \right\}, \label{eq:unsafe_set_semialg}
\end{alignat}
\end{subequations}
where ${\tilde{g}_i\colon \mathbb{R}^n \to \mathbb{R}}$, ${i=1,\ldots,p}$, and ${g_{j}\colon \mathbb{R}^n \to \mathbb{R}}$, ${j=1,\ldots,m}$, are polynomials.
It is important to emphasize that even if the initial and unsafe sets are non-semialgebraic, one can use lightweight algorithms (\eg \cite{dabbene2017simple}) to generate basic semialgebraic sets that overapproximate the non-semialgebraic sets. 
Thus, the SOS optimization in the next proposition can handle initial and unsafe sets with very complicated shapes, possibly non-convex and non-connected. 
\begin{prop}\label{thm:barrier_fcn_SOS}
If there exist a polynomial function ${B\colon \mathbb{R}^n \to \mathbb{R}}$, SOS polynomials\footnote{%
A polynomial is SOS if it can be written as a sum of polynomials, each raised to an even power~\cite{Parrilo2003Semidefiniteprogrammingrelaxations}.} 
${\phi\colon \mathbb{R}^n \to \mathbb{R}}$,
${\tau_{i}\colon \mathbb{R}^n \to \mathbb{R}}$, ${i=1,\ldots,p}$, and ${\sigma_{j}\colon \mathbb{R}^n \to \mathbb{R}}$,  ${j=1,\ldots,m}$, and ${\varepsilon_1, \varepsilon_2 > 0}$ satisfying the following conditions%
\footnote{The terms with ${\varepsilon_2}$ in \eqref{eq:barrier_SOS_conditions} are introduced to improve the numerical stability of the optimization problem and/or to be able to deal with the strict inequalities using SOS.}
\begin{subequations}\label{eq:barrier_SOS_conditions}
\begin{alignat}{2}
&  &&  {-B(\bx)} - \sum_{i=1}^p\nolimits \tau_i(\bx) \, \tilde{g}_i(\bx) \, \text{ is SOS},\label{eq:SOS_B_fcn_cond1}\\
&  && B(\bx) - \varepsilon_2 - \sum_{j=1}^m\nolimits \sigma_j(\bx) \, g_j(\bx) \, \text{ is SOS},\label{eq:SOS_B_fcn_cond2}\\
&  && {-{\nabla{B}(\bx)}^\mathsf{T}} f(\bx)\!-\!\phi(\bx)  (\varepsilon_1\!-\!B^2(\bx))\!-\!\varepsilon_2 \| \bx \|^2 %
    \text{ is SOS},\label{eq:SOS_B_fcn_cond3}
\end{alignat}
\end{subequations}
then $B$ is a barrier certificate for the system $f$ with respect to the initial set $\mathcal{X}_0$ and the unsafe set $\mathcal{X}_u$.
\end{prop}
\begin{proof}
First, suppose that ${\bx \in \mathcal{X}_0}$. Then ${\tilde{g}_i(\bx) \ge 0}$ and this implies that ${\sum_{i=1}^p \tau_i(\bx) \, \tilde{g}_i(\bx) \ge 0}$.
This last inequality and the non-negativity of \eqref{eq:SOS_B_fcn_cond1} imply that ${B(\bx) \le 0}$.
Thus, ${B(\bx) \le 0}$ for all ${\bx \in \mathcal{X}_0}$ and 
\eqref{eq:B_constraint_1} is satisfied. 
A similar argument can be applied to show that \eqref{eq:SOS_B_fcn_cond2} implies $B{(\bx) \ge \varepsilon_2 > 0}$ for all ${\bx \in \mathcal{X}_u}$, fulfilling \eqref{eq:B_constraint_2}.
Finally, if ${|B(\bx)| \le \sqrt{\varepsilon_1}}$ then \eqref{eq:SOS_B_fcn_cond3} implies that ${{\nabla{B}(\bx)}^\mathsf{T} f(\bx) \le {-\varepsilon_2\,\|\bx\|^2} \le 0}$, satisfying \eqref{eq:B_constraint_3}.
\end{proof}
Based on Proposition~\ref{thm:barrier_fcn_SOS}, we propose \emph{ABC-DS} in Algorithm~\ref{alg:ABC-DS}. %
\begin{figure}[!t]%
\vspace*{-0.71em}%
\begin{algorithm}[H]
\caption{ABC-DS}
\label{alg:ABC-DS}
\noindent
\textbf{Input}: Samples of demonstrated positions $\bx^{(i)}_{\mathrm{ref}}$ and corresponding velocities $\dot{\bx}^{(i)}_{\mathrm{ref}}$ from reference trajectory, basic semi-algebraic initial set \eqref{eq:initial_set_semialg}, unsafe set \eqref{eq:unsafe_set_semialg}, and positive scalars $\varepsilon_1$, $\varepsilon_2$ and $\varepsilon_3$. \\
\textbf{Output}: Dynamics $f$, Lyapunov function $V$, and barrier $B$.\\[0.9em]
Search polynomial functions ${f, V, B \colon \mathbb{R}^n\to\mathbb{R}^n}$ and sum-of-squares polynomials ${\phi, (\tau_{i})_{i=1}^{p},(\sigma_{j})_{j=1}^m \colon \mathbb{R}^n \to \mathbb{R}}$ satisfying:
\begin{subequations}\label{eq:DS_barrier_opt_problem}
\begin{alignat}{2}
&\!\min_{f}  &\quad& {\sum_{i=1}^{N}\nolimits \| \dot{\bx}^{(i)}_{\mathrm{ref}} - f(\bx^{(i)}_{\mathrm{ref}}) \|^2} \label{eq:DS_barrier_opt_problem_objective}\\
&\text{\textup{s.t.}}  &&  V(\bx) - \varepsilon_3 \, \| \bx \|^2 %
    \text{ is SOS},\label{eq:DS_barrier_opt_problem_constraint_Vpos}\\
&  &&  {-{\nabla{V}(\bx)}^\mathsf{T}} f(\bx) - \varepsilon_3 \, \| \bx \|^2 %
    \text{ is SOS},\label{eq:DS_barrier_opt_problem_constraint_VLie}\\
&  &&  {\sum_{i=1}^{N}\nolimits \max\left\{0,B(\bx_{\mathrm{ref}}^{(i)})\right\} \le 0}\label{eq:DS_barrier_opt_problem_constraint_demo},\\
&  && \text{and } \eqref{eq:DS_opt_problem_constraints_1} \text{ and }
\eqref{eq:barrier_SOS_conditions} \text{ hold.} \nonumber
\end{alignat}
\end{subequations}
\vspace{-1.3em}
\end{algorithm}
\vspace*{\mycaption - 0.8em}%
\end{figure}
The constraints \eqref{eq:barrier_SOS_conditions} represent the safety requirements, while the constraints \eqref{eq:DS_opt_problem_constraints_1}, \eqref{eq:DS_barrier_opt_problem_constraint_Vpos} and \eqref{eq:DS_barrier_opt_problem_constraint_VLie} represent the global stability of the origin (for more details, see \cite{papachristodoulou2005tutorial}). 
The sum-of-squares constraints can be easily recast as linear and bilinear\footnote{%
Bilinear matrix inequality constraints can be viewed as bilinear extensions of LMIs. Differently from the latter, the former is a non-convex problem. Nevertheless, there are specialized solvers for solving them \cite{kovcvara2006penbmi}, and many of these problems can be approached by bisection techniques~\cite{fukuda2001branch, cunis2023sequential}. For instance, optimization problems with quadratic objectives and BMI constraints have been solved in~\cite{FigueroaBillard2018PhysicalBayesianDS} and~\cite{FigueroaBillard2022Locallyactiveglobally} to learn DSs without obstacles.} matrix inequalities (LMI and BMI, respectively) by using specialized parsers~\cite{Lofberg2009, sostools, weisser2019polynomial}.
The constraint \eqref{eq:DS_barrier_opt_problem_constraint_demo} is used to exclude local optima where the barrier certificate cuts through the reference trajectories and as such, to bias the solver toward a useful solution.
It is worthy to note that this constraint is linear in the variables of $B$ and, thus, adds only marginal complexity to the problem even for large $N$.

As the degree of the polynomial variables increases, the search space of the optimization problem grows, increasing the chance of finding a solution. 
The price of increasing the degree is that it makes the problem more complex
and, for very high degrees, the resulting problem could be numerically intractable for current solvers. Nevertheless, recent advances in the SOS literature have been enhancing and pushing forward the scalability of SOS-based optimization approaches~\cite{Majumdar2020SOSScalability}.
It is also important to note that employing very high-degree polynomials in sum-of-squares optimization problems makes them more prone to numerical issues such as floating-point errors and ill-conditioned problem formulations, which can influence the trustworthiness of the final numerical results to varying degrees~\cite{roux2018validating}.
Nevertheless, post-analysis procedures can be performed to certify that the DSs generated by the optimization problem indeed satisfy the required properties (see \eg \cite{dai2017barrier, basagiannis2023smt}). In this work, the DSs were verified using the SMT solvers cvc5~\cite{barbosa2022cvc5} and dReal~\cite{gao2013dreal}.

\section{Simulations}\label{sec:simulations}

In this section we illustrate the results of this paper by computing DSs for a representative subset of the LASA Handwriting Dataset\footnote{Available at \href{https://bitbucket.org/khansari/lasahandwritingdataset}{https://bitbucket.org/khansari/lasahandwritingdataset}}~\cite{KhansariZadehBillard2011LearningStableNonlinear}.
We subsample the data to 100 samples per reference trajectory and, without loss of generality, we assume the attractor point to coincide with the origin.
When using polynomials of higher degrees, the spread of coefficients of the decision variables in the optimization problem can reach values in the $10^6$-range and above, which may cause the solver to behave unfavorably.
In order to reduce this spread, we normalize the reference trajectories and corresponding reference velocities.
All the simulations are run with MATLAB R2023a, using the numerical solvers PENBMI~\cite{kovcvara2006penbmi} and MOSEK~\cite{mosek} to solve the optimization problems.
First, we show the competitiveness of ABC-DS in encoding diverse motions with polynomial DSs. Then, we show the capability to produce DSs that are safe in the presence of obstacles but still closely resemble the reference trajectories.

\subsection{Comparison of polynomial DSs with LPV-DS}

Using  ABC-DS
without the barrier-related constraints, \ie without \eqref{eq:barrier_SOS_conditions} and \eqref{eq:DS_barrier_opt_problem_constraint_demo}, we are able to generate polynomial DSs that accurately encode many demonstrations of the LASA dataset, including highly nonlinear ones like \texttt{Leaf2} (see Fig.~\ref{fig:sim_poly_perf}).
All shown results are obtained with the same optimization parameters and solver settings; in particular, $\mathrm{deg}(f)=6$, $\mathrm{deg}(V)=4$.
We also warm-start some variables of the optimization problem of Algorithm~\ref{alg:ABC-DS} with the solution of the convex optimization problem obtained by enforcing $\mathrm{deg}(f)=1$ and $V(\bx)=\bx^\mathsf{T} \bx$.

To quantify the overall similarity of a dynamical system to its underlying reference trajectories, we report the mean squared error (MSE) between the reference velocities $\dot{\bx}^{\mathrm{ref}}$ and the corresponding velocities output by the DS, $f(\bx^{\mathrm{ref}})$.
Fig.~\ref{fig:mse_boxchart} gives a statistical overview of the MSE of our method on the scenarios from Fig.~\ref{fig:sim_poly_perf}.
It also conveys the competitive performance of our method compared to LPV-DS.\footnote{We use the LPV-DS code from \href{https://github.com/nbfigueroa/ds-opt/tree/cc62dbc12f47a2360e057423540ae5ead08d8c39}{https://github.com/nbfigueroa/ds-opt} with default GMM fitting from the provided demo script.}

Nevertheless, one current limitation of our approach is that for certain datasets the solver may have difficulty finding feasible solutions.
However, %
in several cases a more extensive hyperparameter search on the solver settings or the optimization parameters along with heuristics for initializing the variables can solve this issue.

\begin{figure}[!t]%
    \vspace*{\mySepBetweenTopAndFig+1em}%
    \centering%
    \setlength{\fboxsep}{-0.25pt}%
    \setlength{\fboxrule}{0.5pt}%
    \def\tmpwidth{1.058843685114173in-1.33333333333\fboxrule}%
    \subfloat{\label{fig:perf_lasa_angle}\fbox{\scalebarimg{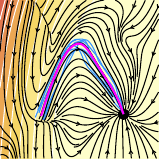}{\tmpwidth}{}{black}}}%
    \hspace{-0.5\fboxrule}\hspace{\fill}\hspace{-0.5pt}\hspace{-0.5\fboxrule}%
    \subfloat{\label{fig:perf_lasa_c}\fbox{\scalebarimg{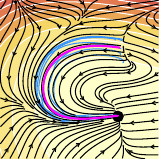}{\tmpwidth}{}{black}}}%
    \hspace{-0.5\fboxrule}\hspace{\fill}\hspace{-0.5pt}\hspace{-0.5\fboxrule}%
    \subfloat{\label{fig:perf_lasa_g}\fbox{\scalebarimg{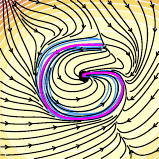}{\tmpwidth}{}{black}}}\\%
    \vspace{-0.7mm}%
    \subfloat{\label{fig:perf_lasa_leaf2}\fbox{\scalebarimg{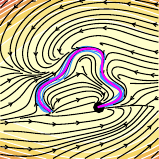}{\tmpwidth}{}{black}}}%
    \hspace{-0.5\fboxrule}\hspace{\fill}\hspace{-0.5pt}\hspace{-0.5\fboxrule}%
    \subfloat{\label{fig:perf_lasa_n}\fbox{\scalebarimg{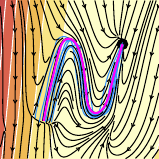}{\tmpwidth}{}{black}}}%
    \hspace{-0.5\fboxrule}\hspace{\fill}\hspace{-0.5pt}\hspace{-0.5\fboxrule}%
    \subfloat{\label{fig:perf_lasa_p}\fbox{\scalebarimg{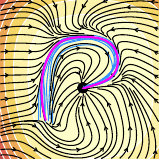}{\tmpwidth}{}{black}}}\\%
    \vspace{-0.7mm}%
    \subfloat{\label{fig:perf_lasa_sine}\fbox{\scalebarimg{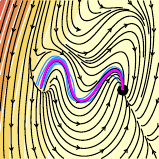}{\tmpwidth}{}{black}}}%
    \hspace{-0.5\fboxrule}\hspace{\fill}\hspace{-0.5pt}\hspace{-0.5\fboxrule}%
    \subfloat{\label{fig:perf_lasa_s}\fbox{\scalebarimg{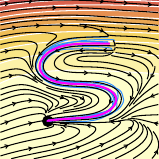}{\tmpwidth}{}{black}}}%
    \hspace{-0.5\fboxrule}\hspace{\fill}\hspace{-0.5pt}\hspace{-0.5\fboxrule}%
    \subfloat{\label{fig:perf_lasa_worm}\fbox{\scalebarimg{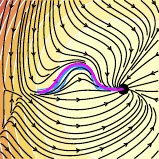}{\tmpwidth}{}{black}}}%
    \vspace*{\mySepBetweenFigAndCap}%
    \caption{%
        Polynomial DS generated by our proposed method (without barrier) on a representative subset of the LASA handwriting dataset. Black stream lines depict the vector field with a globally asymptotically stable equilibrium (black dot). The yellow contours depict the values of the Lyapunov function, where lighter areas indicate lower values. The reference trajectories are shown in blue, and a trajectory starting from the references' mean initial point and simulated by the DS is plotted in pink.%
    }%
    \label{fig:sim_poly_perf}%
    \vspace*{\mycaption}%
\end{figure}

\begin{figure}[!t]
    \vspace*{0.48em}%
    \centering%
    \includegraphics[width=1.0\columnwidth]{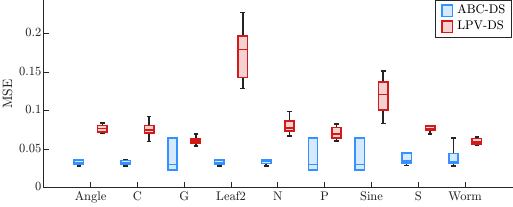}%
    \vspace*{\mySepBetweenFigAndCap}%
    \caption{%
        Mean and standard deviation of the MSE of our polynomial approach (without barrier) versus LPV-DS on a subset of the LASA handwriting dataset. We evaluate each one 10 times with different random seeds for the random number generation, which influences the solver. The reported values are for normalized reference data.%
    }%
    \label{fig:mse_boxchart}%
    \vspace*{\mycaption}%
\end{figure}

\subsection{Encoding demonstrations using ABC-DS}

First, we select four motions from the LASA handwriting dataset and place an obstacle (in the shape of an ellipse) near the reference trajectories in each one of these scenarios. For simplicity, we define any $\mathcal{X}_0$ as a hypersphere large enough to enclose all respective reference trajectories' starting points.
The solutions computed by our approach are plotted in Fig.~\ref{fig:sim_avoidance_roam}. Note that we use slightly different optimization parameters for each scenario, with ${\mathrm{deg}(f) \in \{4, 5\}}$, ${\mathrm{deg}(V) \in \{2, 4\}}$, ${\mathrm{deg}(B) \in \{3, 4\}}$,
and the solver settings differing mostly in the number of iterations.
It can be seen that the trajectories generated by our DS when starting at the reference trajectories' initial points follow the reference trajectories very closely.
At the same time, the computed barrier certificates define a non-conservative safe region: in practice, the certified safe sets (as introduced in Section~\ref{sec:preliminaries}) are not limited to $\mathcal{X}_0$, %
but typically encompass a large region of the state space.

\FloatBarrier

Second, we demonstrate the capability of ABC-DS to work with complex concave obstacles (see Fig.~\ref{fig:abcds_ushape}). Specifically, we select a scenario where the reference trajectories enter the convex hull of such obstacle, because then the concavity cannot simply be mitigated by replacing the obstacle by its convex hull. Therefore, the state-of-the-art method~\cite{huber2022avoiding} cannot be used. Also, ABC-DS generates a DS that reproduces the demonstrations faithfully. It is important to remark that the barrier's 0-level set is fit tightly between demonstrations and obstacle.

\begin{figure}[!t]
    \vspace*{\mySepBetweenTopAndFig}%
    \centering%
    \subfloat{\label{fig:sim_lasa_worm}\scalebarimg{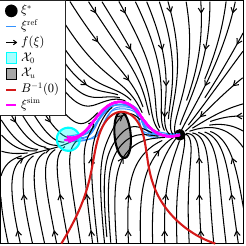}{1.64732064578in}{}{black}}%
    \hspace{\fill}%
    \subfloat{\label{fig:sim_lasa_p}\scalebarimg{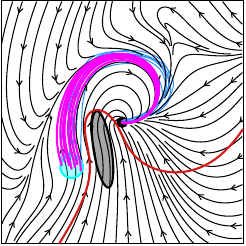}{1.64732064578in}{}{black}}\\%
    \vspace{-0.8mm}%
    \subfloat{\label{fig:sim_lasa_s}\scalebarimg{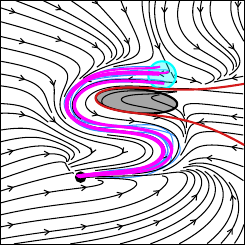}{1.64332064578in}{}{black}}%
    \hspace{\fill}%
    \subfloat{\label{fig:sim_lasa_leaf}\scalebarimg{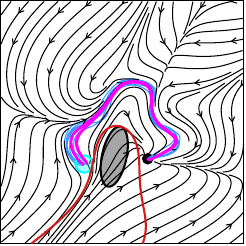}{1.64432064578in}{}{black}}%
    \vspace*{\mySepBetweenFigAndCap}%
    \caption{%
        Performance of ABC-DS on different data from the LASA dataset in the presence of elliptical unsafe sets. The obstacles are safely avoided with non-conservative barrier certificates and the trajectories simulated with the computed DS (black stream lines) are still close to the reference trajectories.%
    }%
    \label{fig:sim_avoidance_roam}%
    \vspace*{\mycaption}%
\end{figure}

\begin{figure}[!t]
    \vspace*{1.05em}%
    \centering%
    \includegraphics[width=1.0\columnwidth]{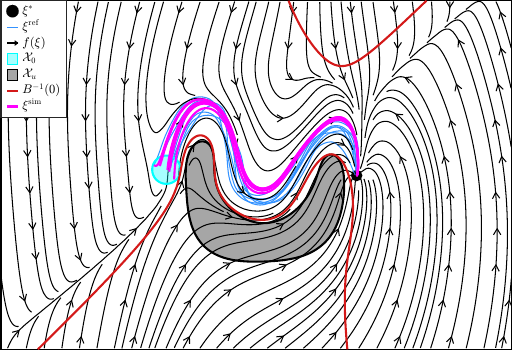}%
    \vspace*{\mySepBetweenFigAndCap}%
    \caption{%
        ABC-DS works with non-star shaped obstacles (black). Here we exemplify the performance in a scenario where the reference trajectories penetrate the convex hull of such obstacle. 
        The solution is obtained %
        by setting ${\mathrm{deg}(f)=5}$, ${\mathrm{deg}(V)=2}$, and ${\mathrm{deg}(B)=4}$.%
    }%
    \label{fig:abcds_ushape}%
    \vspace*{\mycaption - 0.1em}%
\end{figure}

\section{Experiments}\label{sec:experiments}

For this experiment,\footnote{%
Experiments are conducted with a 7-DoF Franka Emika robot~\cite{haddadin2022franka}. 
The robot is controlled using the Franka Control Interface (FCI) at 1 kHz. The control loop runs on a computer (Intel Core i7-10700 @ 2.90GHz) installed with Ubuntu 18.04 LTS and real-time kernel (5.4.138-rt62).
}
the goal is to learn a non-trivial S-shaped motion in the presence of two static box-like obstacles, as illustrated in Figs.~\ref{fig:visual_abstract} and~\ref{fig:ExperimentSetup}. The motion passes through a narrow passage formed by these obstacles.
A video demonstration is available at \href{https://youtu.be/iei-hVL-pfc}{https://youtu.be/iei-hVL-pfc}.

\paragraph*{Recording of reference trajectories and computing the DS}

To obtain the reference trajectories, we use Cartesian impedance control with stiffness in X and Y axis set to 0, while stiffness in the Z axis is set to 500 N/m.
This allows us to move the robot around in the X-Y plane by hand and thereby guide it along the desired path.
Utilizing the joints' position sensors and forward kinematics, we record 14 similar trajectories. For each, the robot is initially positioned randomly within a small circular area.
The trajectories are subsampled to 100 points each.
The obstacles' corner points in the plane are measured by hand. We obtain a basic semialgebraic set description for the unsafe set by first enlarging the polygons---to roughly account for the gripper dimensions---and then fitting a polynomial.
First, we initialize certain variables in the optimization problem presented in Algorithm~\ref{alg:ABC-DS} using the solution derived from the convex optimization problem where we enforce $\mathrm{deg}(f)=1$ and  $V(\bx)=\bx^\mathsf{T} \bx$.
Then, using $\mathrm{deg}(f)=5$, $\mathrm{deg}(V)=2$, $\mathrm{deg}(B)=4$, and $\mathcal{X}_0$ as a hypersphere enclosing all reference trajectories' starting points, our ABC-DS algorithm is able to find a good feasible solution, as plotted in Fig.~\ref{fig:robot_barrier_detailed}.
Similar to scenarios with a single obstacle, the certified safe set $\mathcal{X}_s$ encompasses a non-conservative portion of the state space.

\paragraph*{Control with DS}

To control the robot with the computed DS, we deploy Cartesian impedance control (for X and Y axis) with forward integration. Rudimentary tuning of a few parameters such as the robot impedance is performed.
From Figs.~\ref{fig:robot_barrier_detailed} and~\ref{fig:ExperimentSetup} it can be seen that our DS does indeed keep the robot close to the desired trajectory and also respects the obstacles, thereby demonstrating the effectiveness of ABC-DS in practice.

\begin{figure}[!t]%
    \vspace*{0.48em}%
    \centering%
    \includegraphics[clip, trim=0mm 0.1mm 0.2mm 0mm, width=0.999\columnwidth]{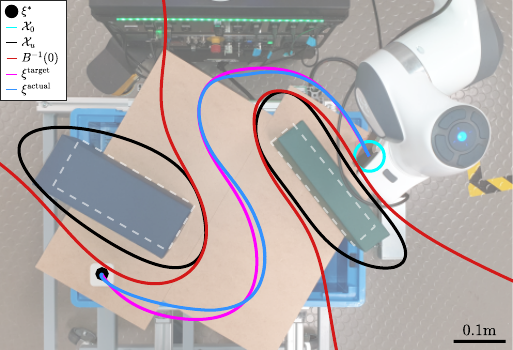}%
    \vspace*{\mySepBetweenFigAndCap}%
    \caption{%
        The robot is shown at a position in the initial set (cyan), about to move along in an S-shaped motion towards the globally asymptotically stable equilibrium (black point). The dark green and dark blue boxes are the obstacles (outlines on the plane are highlighted in white). Overlaid is the derived unsafe set (black). The computed barrier certificate's 0-level set is plotted in red. The pink trajectory is obtained by integrating the DS from the initial position; and blue shows the robot's actual path when controlled by the DS, which closely follows the former.%
    }%
    \label{fig:ExperimentSetup}%
    \vspace*{\mycaption}%
\end{figure}

\paragraph*{Recovery after perturbations}

The final experiment combines compliant behavior, handling of perturbations, obstacle avoidance, and maintaining similarity to a simulated target trajectory.
Here, we deploy passive interaction control~\cite{KronanderBillard2016PassiveInteractionControl}.
We demonstrate that the behavior during perturbations is as desired.
Specifically, when pushed to anywhere in the safe region of the state space, \ie where ${B(\bx) < 0}$, the robot correctly follows the underlying DS which (a) leads it to the globally stable equilibrium, and (b) guides it along paths that never cross the 0-level set of the barrier certificate.

\section{Conclusion}\label{sec:conc}

Our paper introduces ABC-DS, a novel approach that generates stable and safe DSs by solving a SOS optimization problem that simultaneously maximizes the encoding quality of user demonstrations in the DSs while also generating barrier certificates that guarantee obstacle avoidance in certified safe regions of the DSs.
By using basic semialgebraic descriptions for obstacles, ABC-DS is able to deal with non star-shaped obstacles and obstacles with holes. 
In future work, we will investigate applying compositionality results, \eg \cite{Jagtap, dissertation}, to our framework to further increase the scalability of our results.

\balance

\bibliographystyle{ieeetr}
\setlength{\baselineskip}{0pt}

\end{document}